\documentclass{article}
\usepackage{amsmath,amsthm,amssymb}
\usepackage{xcolor}
\usepackage{subcaption}

\usepackage{amsmath,amsfonts,amssymb}

\usepackage{booktabs}
\usepackage{multirow}
\usepackage{cellspace}
\usepackage{setspace}
\usepackage{bm}
\usepackage{bbm}
\usepackage{graphicx}
\usepackage{tikz,pgfplots}
\usepackage{framed}
\usepackage[utf8]{inputenc}
\usepackage[T1]{fontenc}
\usepackage[round]{natbib}

\usepackage{hyperref}
\hypersetup{colorlinks,
            linkcolor=blue,
            citecolor=blue,
            urlcolor=magenta,
            linktocpage,
            plainpages=false}

\usepackage[capitalise]{cleveref}

\theoremstyle{plain}
\newtheorem{theorem}{Theorem}
\newtheorem{lemma}[theorem]{Lemma}

\newtheorem*{theorem*}{Theorem}
\newtheorem*{lemma*}{Lemma}
\newtheorem*{corollary*}{Corollary}
\newtheorem*{proposition*}{Proposition}
\newtheorem*{claim*}{Claim}
\newtheorem*{fact*}{Fact}

\theoremstyle{definition}

\newtheorem*{definition*}{Definition}

\newtheorem*{remark*}{Remark}

\newtheorem*{example*}{Example}

\newcommand{\ignore}[1]{}

\DeclareMathOperator*{\argmin}{arg\,min}

\newcommand{\be}{\begin{align}}
\newcommand{\en}{\end{align}}

\newcommand{\ben}{\begin{align*}}
\newcommand{\enn}{\end{align*}}

\newcommand{\norm}[1]{\left\|#1\right\|_2}

\newcommand{\reals}{\mathbb{R}}

\def\reals{{\mathcal R}}

\newcommand{\D}{{\mathcal{D}}}

\def\reals{{\mathbb R}}

\def\bold0{\mathbf{0}}

\def\be{\mathbf{e}}

\newcommand{\eps}{\varepsilon}

\newcommand{\tO}{\wt{\O}}

\newcommand{\regret}{\text{Regret}}

\newcommand{\X}{\mathcal{X}}

\newcommand{\tf}{\tilde{f}}

\newcommand{\F}{\mathcal{F}}

\newcommand{\tell}{\tilde{\ell}}

\newcommand{\sumin}{\sum_{i=1}^n}
\newcommand{\sumtt}{\sum_{t=1}^T}

\newcommand{\expval}[1]{\mathbb{E}\left[ #1 \right]}

\newcommand{\Var}{\text{Var}}

\newenvironment{proofarg}[1]{%
  \renewcommand{\proofname}{Proof #1}\proof}{\endproof}

\usepackage{mathtools}

\newcommand{\startfoo}{%
    \par\medskip
    \begin{mdframed}[linewidth=1pt]%
    \let\figure\figurehere
    \let\endfigure\endfigurehere
    \ignorespaces
}
\newcommand{\stopfoo}{%
    \unskip
    \end{mdframed}%
    \par\medskip
}

\renewcommand{\tO}{\tilde{\mathcal{O}}}

\usepackage{microtype}
\usepackage{booktabs} %

\usepackage[accepted]{icml2019}

\icmltitlerunning{Online Variance Reduction with Mixtures}

\begin{document}

\twocolumn[
\icmltitle{Online Variance Reduction with Mixtures}

\icmlsetsymbol{equal}{*}

\begin{icmlauthorlist}
\icmlauthor{Zalán Borsos}{eth}
\icmlauthor{Sebastian Curi}{eth}
\icmlauthor{Kfir Y. Levy}{eth}
\icmlauthor{Andreas Krause}{eth}
\end{icmlauthorlist}

\icmlaffiliation{eth}{Department of Computer Science, ETH Zurich}
\icmlcorrespondingauthor{Zalán Borsos}{zalan.borsos@inf.ethz.ch}

\vskip 0.3in
]

\printAffiliationsAndNotice{}  %

\begin{abstract}
Adaptive importance sampling for stochastic optimization is a promising approach that offers improved convergence through variance reduction. In this work, we propose a new framework for variance reduction that enables the use of mixtures over predefined sampling distributions, which can naturally encode prior knowledge about the data. While these sampling distributions are fixed, the mixture weights are adapted during the optimization process. We propose VRM, a novel and efficient adaptive scheme that asymptotically recovers the best mixture weights in hindsight and can also accommodate sampling distributions over sets of points. We empirically demonstrate the versatility of VRM in a range of applications.
\end{abstract}

\section{Introduction}
In the framework of Empirical Risk Minimization (ERM), we are provided with a set of samples $\D=\{x_1,\ldots,x_n\}\subset \X$ drawn from the underlying data distribution, and our goal is to minimize the empirical risk.
Sequential ERM solvers (e.g. SGD, SVRG, etc.) proceed in multiple passes over the dataset and usually require an unbiased estimate of the loss in each round of the optimization. Typically, the estimate is generated by sampling uniformly from the dataset, which is oblivious to the fact that different points can affect the optimization differently. This ignorance can hinder the performance of the optimizer due to the high variance of the obtained estimates.

A promising direction that has recently received  increased interest is represented by (adaptive) importance sampling techniques. Clever sampling distributions can account for  characteristics of datapoints relevant to the optimization in order to improve the performance (see e.g., \cite{zhao2015stochastic,pmlr-v70-namkoong17a, pmlr-v80-katharopoulos18a}).
\paragraph{Why mixtures?} The majority of existing works on   adaptive sampling distributions are \emph{unable to exploit similarities} between points. Thus, only after \emph{several passes over the dataset} do these methods become effective.
This  can become a  major bottleneck for large datasets.

Fortunately, in many situations, it is possible to exploit the structure present in data in some form of a prior. One way of capturing prior knowledge in the framework of importance sampling for variance reduction  is  to propose  plausible fixed sampling distributions before performing the optimization. This is very natural for problems where \emph{similar objects} can be \emph{grouped together}, e.g., based on the class label or clustering in feature space.
In such cases it is  sensible to employ  standard sampling distributions that  draw similar members with the same probability, e.g., as considered by  \citet{zhao2014accelerating}.
Another option is to employ sampling distributions that encourage  diverse sets of samples, e.g., Determinantal Point Processes
\cite{kulesza2012determinantal}.

Suppose that several such proposals for sampling distributions are available. A natural idea is to combine them into a mixture and \emph{adapt the mixture weights} during the optimization process, in order to achieve variance reduction. This setup has the potential of being much more efficient than learning an arbitrary sampling distribution over individual points, provided that one can propose plausible sampling distributions prior to the optimization. Another advantage of this setting is that it enables to \emph{efficiently}  handle distributions not only on individual points, but also on \emph{sets of points}. 
While in principle this can still be treated using existing approaches,  their computational complexity in this case will increase proportionally to the number of possible sets. The latter  often grows \emph{exponentially}  with the size of the sets.

In this work, we develop an online learning approach to variance reduction and ask the following question: given $k$ fixed sampling distributions, how can we choose the mixture weights in order to achieve the largest reduction in the variance of the estimates? We provide a simple yet efficient algorithm for doing so. As for our main contributions, we:
\vspace{-3mm}
\begin{itemize}
\setlength\itemsep{0.08em}
\item formulate the task of adaptive importance sampling for variance reduction with mixtures as an online learning problem,
\item propose a novel algorithm for this setting  with sublinear regret of $\tilde{\mathcal{O}}(T^{4/5})$,
\item substantiate our findings experimentally with accompanying efficient implementation.
\end{itemize}

\textbf{Related Work:}~
There is a large body of work on employing importance sampling distributions for  variance reduction.
 Prior knowledge of gradient norm bounds on each datapoint has been utilized for fixed importance sampling by \citet{needell2014stochastic} and  \citet{zhao2015stochastic}. Adaptive strategies were presented by  \citet{bouchard2015online}, who propose parametric importance sampling distributions  where the parameters of the distributions are updated during the course of the optimization. \citet{NIPS2017_7025} derive a safe adaptive sampling scheme that is guaranteed to outperform any fixed sampling strategy. A significant body of work is concerned with non-uniform sampling of coordinates  in coordinate descent \citep{allen2016even,perekrestenko2017faster, NIPS2018_8137}. All the works presented above provide importance sampling schemes over points or coordinates. Sampling over sets of points and exploiting similarities  between points in these works remains an open question.

Importance sampling has found applications in optimizing deep neural networks. \citet{NIPS2018_7957} and \citet{pmlr-v80-katharopoulos18a} propose methods for choosing the importance sampling distributions over points  proportional to their corresponding approximate gradient norm bounds. 
\citet{NIPS2018_7957} also propose to adapt the learning rate based on the gains in gradient norm reductions. 
 \citet{loshchilov2015online}  propose sampling based on the latest known loss value with exponentially decaying selection probability on the rank. In the context of reinforcement learning, \citet{Schaul2016} suggest a smoothed importance sampling scheme of experiences present in the replay buffer,  based on the last observed TD-error.

Most closely related to our setting, importance sampling for variance reduction has been considered through the lens of online learning in the recent works of \citet{pmlr-v70-namkoong17a}, \citet{salehi2017} and \citet{pmlr-v75-borsos18a}. These works pose the ERM solver as an \emph{adversary} responsible for generating the losses.  The goal of the learner (player) is to minimize the cumulative variance by
 choosing the sampling distributions  adaptively based on partial feedback.
 However, similarly to existing work on adaptive sampling, these methods are not designed for exploiting similarities between points.

\section{Problem Setup} \label{sec:setup}

In the framework of ERM, the goal is to minimize the empirical risk,
\begin{equation*} \label{eq:ERM}
  \min_{\theta\in\Theta }\mathcal{L}(\theta) = \min_{\theta\in\Theta } \frac{1}{n}  \sum_{i=1}^n \ell (x_i, \theta),
\end{equation*}
where $\ell: \mathcal{X} \times \Theta \mapsto \reals_{\geq 0}$ is the loss function, and $\Theta\subseteq \reals^d$ is a compact domain.
A typical sequential ERM solver will run over $T$ rounds and update the parameters based on an \emph{unbiased estimate} $\tilde{\mathcal{L}}_t$ of the empirical loss in each round $t\in [T]$. A common approach for producing  $\tilde{\mathcal{L}}_t$ is to sample a point $i_t \in \{ 1,...,n\}$ uniformly at random, thus ignoring the underlying structure of the data. However, using importance sampling, we can produce these estimates by sampling with \emph{any} distribution $q \in \Delta_n$, where $\Delta_n$ is the $n$-dimensional probability simplex, provided that we compensate for the bias through importance weights.

Suppose we are provided with $k$ sampling distributions $p_1, ..., p_k \in \Delta_n$. We combine these distributions into a \emph{mixture}, in which the probability of sampling $x_i$ is given by $ w^\intercal p(i)$, where $w\in \Delta_k$ is the mixture weight vector and $p(i):= [p_1(i), ..., p_k(i)]$. Using the mixture, we obtain the loss estimate
\begin{equation*}
\tilde{\mathcal{L}}_t(\theta)= r_i \cdot \ell (x_i, \theta),
\end{equation*}
where $r_i=\frac{1}{n \cdot w^\intercal p(i)}$ is the importance weight of point $i$.

The performance of solvers such as SGD, SAGA \citep{defazio2014saga} and SVRG \citep{johnson2013accelerating} is known to improve when the variance of $\tilde{\mathcal{L}}_t$ is smaller. Thus, a  natural performance measure for our mixture sampling distribution is the \emph{cumulative variance} of  $\tilde{L}_t$ through the $T$ rounds of optimization,
\begin{equation*}
\sumtt\text{Var}_q(\tilde{\mathcal{L}}(\theta_t)) = \frac{1}{n^2} \sumtt \sum_{i=1}^n \frac{\ell^2 (x_i, \theta_t)}{w^\intercal p(i)} - \sumtt \mathcal{L}^2(\theta_t).
\end{equation*}
Since only the cumulative second moments depend on $w$, we define our cost function at time $t$ as $\frac{1}{n^2}f_t(w)$, where
\begin{equation*}
f_t(w) = \sumin \frac{\ell_t^2(i)}{w^\intercal p(i)},
\end{equation*}
and where  we have introduced the shorthand $\ell_t(i):=\ell(x_i, \theta_t)$.
Through the lens of online learning, it is natural to regard the sequential solver as an \emph{adversary} responsible for generating the losses $\{ \ell_t \}_{t \in [T]}$ and to measure the
performance using the notion of the \emph{cumulative regret},
\begin{equation*}
\regret_T = \frac{1}{n^2} \left( \sumtt f_t(w_t) -   \min _{w \in \Delta_k} \sumtt  f_t(w) \right).
\end{equation*}
By devising a no-regret algorithm, we are guaranteed to compete asymptotically with the best mixture weights in hindsight. The online variance reduction with mixtures (OVRM) protocol is presented in Figure  \ref{alg:OVRM}. 

Motivated by empirical insights, we impose a natural mild restriction on our setting, which is easily verified  in practice:
\paragraph{Assumption 1}Throughout the work, we assume that  the losses are bounded, $\ell_t^2(i) \leq L$ for all $t \in [T]$, $i \in [n]$ and  that all mixture components  place a probability mass at most $p_{\max} = \frac{c}{n}$  on any specific point, where $c \in [1,n]$. That is, $p_j(i) \leq \frac{c}{n}$, for all $i \in [n], j \in [k]$.

The choice of $c$ is in the hands of the designer who determines the fixed sampling distributions in the mixture. Although $c$ can be as large as $n$, in our experiments, we show how we can obtain a large speedup in the optimization due to the reduced variance  by using mixtures with small values of $c$ (the maximal value of $c$ is less than 50 in the experiments). We note that our setup also allows choosing $k=n$, as many mixture components as points, where a mixture puts all its probability mass on a specific point, i.e. $p_j(i) = \delta_{ij}$ for $i \in [n], j \in [k]$ and consequently $c=n$. Under this perspective, our setting is a \emph{strict generalization} of adaptively choosing sampling distributions over the points, which is the main objective of  \citet{salehi2017}, \citet{pmlr-v70-namkoong17a} and \citet{pmlr-v75-borsos18a}. However, in practical scenarios, $k$ is usually small due to the limited number of available proposal distributions. 

\begin{figure}[h]
\begin{framed}
\centering{ \textbf{OVRM Protocol}\\}
 \flushleft
 \textbf{Input}: Dataset $\mathcal{D} = \{ x_1, ..., x_n\}$, sampling distributions $p=[p_1,...,p_k]$. \\
   \textbf{for} $t=1,\ldots, T$ \textbf{do} \\
        \quad player chooses $w_t \in \Delta_k $ \\
        \quad adversary chooses  $\ell_t \in \mathbb{R}^n$  \\
        \quad player draws $I_t \sim w_t^\intercal p$\\
        \quad player incurs cost $f_t(w_t)/n^2$ and receives $\ell_t(I_t)$ as \\
        \qquad partial  feedback \\
   \textbf{end for} \\
\end{framed}
   \vspace{-2mm}
\caption{Online variance reduction protocol with mixtures and partial feedback.}
\label{alg:OVRM}
\end{figure}

We have formulated our objective as minimizing the cumulative second moment of the loss estimates. If we choose to substitute $\ell_t(i)$ with $\lVert \nabla \ell (x_i, \theta_t) \rVert $, the norm of the loss gradients, the corresponding cumulative second moment has a stark relationship to the quality of optimization --- for example, this quantity directly appears in the regret bounds of AdaGrad \citep{duchi2011adaptive}. For a more detailed discussion see \citet{pmlr-v75-borsos18a}.

Let us discuss some properties of our setting. Since $f_1,...,f_T$ are convex functions on $\Delta_k$, the problem is an instance of online convex optimization (OCO). While the OCO framework offers a wide range of well-understood tools, our biggest challenge here is posed by the fact that the cost functions are \emph{unbounded}, together with the fact that we have \emph{partial feedback}. The majority of existing regret analyses assume boundedness of the cost functions. 

For simplicity, we focus on choosing \emph{datapoints}; nevertheless, our method applies to choosing coordinates or blocks of \emph{coordinates} in coordinate descent and  can work on top of any sequential solver that builds on unbiased loss estimates. As we will see in Section \ref{sec:partial_info}, the complexity and the  performance guarantee of our algorithm is \emph{independent of $n$}, which broadens its applicability significantly. For example, instead of learning mixtures of distributions over points, we can learn a mixture for variance-reduced sampling of  \emph{minibatches}, where each mixture component is a fixed $k$-Determinantal Point Process \citep{kulesza2012determinantal}.

\section{Full Information Setting}
Let us assume for the moment that in each round of Protocol \ref{alg:OVRM}, the player receives full information feedback, i.e., sees the losses associated to \emph{all} points $[\ell_t(1), ...,\ell_t(n)]$ instead of observing only the loss $\ell_t(I_t)$ associated with  the chosen point. This setup, referred to as full information setting, is unrealistic, yet it serves as the main tool for the analysis of the partial information setting, which we discuss in Section \ref{sec:partial_info}.
Here, we first show an efficient algorithm for the full information setting (Alg.~\ref{alg:ons}), ensuring  a regret bound of
$\tO(k^{1/2}T^{2/3})$.

Unfortunately, even under Assumption 1, our cost function can be unbounded.
In order to tackle this, we consider that the last mixture component (the $k$-th one) is  always the uniform distribution. If this is not the case in practice, we can simply attach the uniform distribution to the given sampling distributions. This is w.l.o.g.,  since the optimal $w$ in hindsight is allowed to assign 0 weight on any component. Thus, we have that  $p(i) = [p_1(i),..., p_{k-1}(i), 1/n]$ for all $i \in [n]$. For the analysis, we consider the restricted simplex $\Delta'_k=\{ w \in \Delta_k | \, w(k) \geq \gamma \}$,  where the last weight corresponding to the uniform component is larger  than some $\gamma \in (0, 1]$. This allows for decomposing the regret as follows:
\begin{align}\label{eq:regret-decomp}
\regret_T &= \underbrace{ \frac{1}{n^2} \left(\sumtt f_t(w_t) -   \min_{w \in \Delta'_k}\sumtt  f_t(w) \right) }_{(A)} \\
&+  \underbrace{ \frac{1}{n^2}   \left( \min_{w \in \Delta'_k}\sumtt  f_t(w) - \min_{w \in \Delta_k}\sumtt  f_t(w) \right)  }_{(B)}. \nonumber
\end{align}

This decomposition introduces a trade-off. By choosing larger $\gamma$,  we pay more in term $(B)$ for potentially missing the optimal $w$. Nevertheless, larger $\gamma$ makes the cost function ``nicer'':  not only does it reduce the upper bounds on the costs, but it also turns $f_t$ into an \emph{exp-concave} function, as we will later show.

First, let us focus on bounding $(B)$, which captures the excess regret of working in $\Delta'_k$ instead of $\Delta_k$.
\begin{lemma} \label{lemma:b-bound} The reduction to the restricted simplex $\Delta'_k$ incurs the excess regret of
\begin{equation*}
(B)  \leq \gamma LT.
\end{equation*}
\end{lemma}
\begin{proof}
 Let $w_*=\arg \min_{w \in \Delta_k}\sumtt  f_t(w)$.  Let $w_*' = (1-\gamma)w_* + \gamma 
e_k$, where $e_k=[0,...,0,1]$. Now clearly $w_*' \in \Delta'_k$. We can observe that for all $i \in [n]$,
\begin{align*}
\frac{1}{w_*'^\intercal p(i) } - \frac{1}{w_*^\intercal p(i) } =  \frac{\gamma (w_* - e_k)^\intercal p(i)}{w_*'^\intercal p(i)  \cdot w_*^\intercal p(i)} 
 = \frac{\gamma \left(w_*^\intercal p(i) - \frac{1}{n}\right)}{w_*'^\intercal p(i)  \cdot w_*^\intercal p(i)}.
\end{align*}
If for some $i$ we have $w_*^\intercal p(i) - 1/n < 0$, or, equivalently $w_*^\intercal p(i) < 1/n$, we can ignore this specific term. Otherwise, if $w_*^\intercal p(i) \geq 1/n$, then also evidently  $w_*'^\intercal p(i) \geq 1/n$. Denote $\mathbb{I}_+$ the set of $i$'s for which $w_*^\intercal p(i) \geq 1/n$. Using the previous observations, we can now bound $(B)$:
\begin{align*}
n^2  \cdot  (B) &\leq  \sumtt \sum_{i \in \mathbb{I}_+ }\frac{\gamma \ell_t^2(i) \left(w_*^\intercal p(i) - \frac{1}{n}\right)}{w_*'^\intercal p(i)  \cdot w_*^\intercal p(i)}\\
& \leq \gamma L \sumtt \sum_{i \in \mathbb{I}_+ } \frac{w_*^\intercal p(i)}{w_*'^\intercal p(i)  \cdot w_*^\intercal p(i)} \leq n^2 \gamma L T,
\end{align*}
where the last inequality uses the fact that $w_*'^\intercal p(i) \geq 1/n$ for all $i \in \mathbb{I}_+$ and that $|\mathbb{I}_+| \leq n$. This proves the claim.
\end{proof}

By constraining our convex set to the  restricted simplex $\Delta_k'$, we achieve desirable properties of our cost function: $f_t$ and its gradient norm are bounded. The first natural option for solving the problem is Online Gradient Descent (OGD). However, OGD can only guarantee a $\mathcal{O}(\sqrt{T})$ bound on $(A)$ --- we elaborate on this in the supplementary material. We can obtain better regret bounds by noticing that  restricting the domain to $\Delta_k'$ has another advantage: it allows for exploiting curvature information as $f_t$ is \emph{exp-concave} on this domain.

A convex function $g: \mathcal{K} \mapsto \mathbb{R}$, where $\mathcal{K}$ is a convex set, is called $\alpha$-exp-concave, if $e^{-\alpha g(x)}$ is concave. Exp-concavity is a weaker property than strong convexity, but it can still be exploited to achieve logarithmic regret bounds \cite{hazan2006logarithmic}. In the following result, we establish the exp-concavity of our cost function on the restricted simplex $\Delta'_k$.

\begin{lemma}\label{lemma:exp-concave}
$f_t$ is $\frac{2\gamma}{n^2 L}$-exp-concave on $\Delta_k'$ for all $t\in [T]$.
\end{lemma}
\begin{proofarg}{sketch}
In order to prove exp-concavity, we rely on the following result \cite{hazan2016introduction}:  a twice differentiable function $g$ is $\alpha$-exp-concave at $x$, iff
\begin{equation} \label{eq:exp-concavity}
\nabla^2 g(x) \succeq \alpha \nabla g(x) \nabla^\intercal g(x).
\end{equation}
In our case, $\mathcal{K} = \Delta'_k$ and $\nabla f_t(w) = -\sumin \frac{\ell_t^2(i)p(i)}{(w^\intercal p(i))^2}$ and $\nabla^2 f_t(w) = 2 \sumin \frac{\ell_t^2(i)p(i)p(i)^\intercal}{(w^\intercal p(i))^3}$.  We can prove the property of exp-concavity using the following observation: for $x_1,...,x_n \in \mathbb{R}^d$ we have
\begin{equation} \label{eq:matrix-jensen}
\left(\sumin x_i \right) \left(\sumin x_i \right)^\intercal \preceq n \sumin x_i x_i^\intercal,
\end{equation}
which is a result of the definition of positive semi-definiteness and Jensen's inequality.
If we instantiate $x_i :=-\frac{\ell_t^2(i)p(i)}{(w^\intercal p(i))^2}$ and plug $f_t$ into Equation \ref{eq:exp-concavity}, we can identify $\alpha = 2\gamma/(n^2 L)$ after an additional step of lower bounding $w^\intercal p(i)$ by $\gamma /n$. From the last step, we can see that working in the restricted simplex $\Delta'_k$ is crucial for achieving exp-concavity.
\end{proofarg}

\begin{algorithm}[h]
  \caption{ONS}\label{alg:ons}
  \begin{algorithmic}[1]
  	\INPUT Dataset $\mathcal{D} = \{ x_1, ..., x_n\}$, sampling distributions $p=[p_1,...,p_{k-1}, 1/n]$, parameters $\gamma$, $\beta$, $\eps \geq 0$.
   \STATE $w_1 = [1/k, ..., 1/k]$
   \STATE $H_0 = \eps \mathbb{I}$ 
    \FOR{$t$ in 1 to $T$}    
       		 \STATE play $w_t$,  observe $f_t(w_t)$
        \STATE update: $H_t = H_{t-1} + \nabla f_t(w_t) \nabla^\intercal f_t(w_t)$
        \STATE Newton step: $w' =w_t - \frac{1}{\beta}H_t^{-1} \nabla f_t(w_t)$
        \STATE project: $w_{t+1} = \argmin_{w\in \Delta'_k} (w-w')^\intercal H_t (w-w')$
    \ENDFOR 
  \end{algorithmic}
\end{algorithm}

Since the $f_t$'s are  $\alpha$-exp-concave functions in the restricted simplex, we can bound $(A)$ by employing Algorithm \ref{alg:ons}, known as Online Newton Step (ONS), which provides the following guarantee  for appropriately chosen $\beta$ and $\eps$ \citep{hazan2006logarithmic}:
 
\begin{equation} \label{eq:ons-bound}
\sumtt f_t(w_t) -   \min_{w \in \Delta'_k}\sumtt  f_t(w) \leq 5 \left(\frac{1}{\alpha}  + GD \right) k \log T,
\end{equation}
where $D=\sqrt{2}$ is the diameter of $\Delta_k'$ and $G\geq \sup_{w \in \Delta'_k, t \in [T]} \norm{\nabla f_t (w)}$ is an upper bound on the gradient norm:
\begin{align*}
\norm{\nabla f_t (w)}&= \norm {\sumin \frac{\ell_t^2(i)p(i)}{(w^\intercal p(i))^2}} \leq \frac{n^2L}{\gamma^2} \norm{\sumin p(i)} \\
&= \frac{n^2L}{\gamma^2} \norm{(1,...,1)} =\frac{n^2L\sqrt{k}}{\gamma^2}=:G,
\end{align*}
where the inequality uses that in $\Delta'_k$ we have $w^\intercal p(i) \geq \gamma/n$. Using these bounds together with Lemma \ref{lemma:exp-concave} in Equation \ref{eq:ons-bound}, we can finally bound $(A)$:
\begin{lemma} \label{lemma:a-bound}
Algorithm \ref{alg:ons} ensures
\begin{equation*}
(A) \leq \frac{10L k^{3/2} \log T}{\gamma^2}.
\end{equation*}
\end{lemma}

Finally, we can combine the results from Lemma \ref{lemma:a-bound} and \ref{lemma:b-bound} and optimize the parameter $\gamma$ that controls the trade-off, to arrive at the following regret bound with respect to the full simplex:

\begin{theorem} \label{thm:main}
The regret of Algorithm \ref{alg:ons} is
\begin{equation*}
\regret_T \leq 5 L k^{1/2} T^{2/3}\log^{1/3} T. 
\end{equation*}
\end{theorem}

\section{The Partial Information Setting}
\label{sec:partial_info}
In a practice, the player only receives partial feedback from the environment corresponding to the loss of the chosen point, as presented in Figure  \ref{alg:OVRM}.  Even under partial feedback, the unbiasedness of the loss estimates must be ensured. For this, we propose our main algorithm, \emph{Variance Reduction with Mixtures (VRM)}, presented in  Algorithm \ref{alg:ons-bandit}. VRM is inspired by the seminal work of \citet{auer2002nonstochastic}, in its approach to  obtaining unbiased estimates under partial information. The 
  algorithm in line 4 samples $I_t \sim w_t^\intercal p$ and receives only $\ell_t(I_t)$ as feedback in round $t$. We obtain an estimate by
\begin{equation} \label{eq:mod-loss}
\tilde{\ell}_t^2(i) =  \frac{\ell_t^2(i)}{w_t^\intercal p(i)} \cdot \mathbbm{1}_{I_t=i},
\end{equation}
which is clearly unbiased due to $\expval{\tilde{\ell}_t^2(i) |\ell_t, w_t } = \ell_t^2(i)$. We can  analogously define $\tilde{f}_t(w) = \sumin \frac{\tilde{\ell}_t^2(i)}{w^\intercal p(i)}$. With this choice, the estimates can be readily used, similar to the full information setting, in Algorithm \ref{alg:ons-bandit}.

\begin{algorithm}[h]
  \caption{VRM}\label{alg:ons-bandit}
  \begin{algorithmic}[1]
  	\INPUT Dataset $\mathcal{D} = \{ x_1, ..., x_n\}$, sampling distributions $p=[p_1,...,p_{k-1}, 1/n]$, parameters $\gamma$, $\beta$, $\eps \geq 0$.
   \STATE $w_1 = [1/k, ..., 1/k]$
   \STATE $H_0 = \eps \mathbb{I}$ 
    \FOR{$t$ in 1 to $T$}    
       		 \STATE sample $I_t \sim w_t^\intercal p$,  receive $\ell_t(I_t)$, set $\tilde{f}_t(w_t)= \frac{\ell_t^2(I_t)}{(w_t^\intercal p(I_t))^2}$
        \STATE update: $H_t = H_{t-1} + \nabla \tilde{f}_t(w_t) \nabla^\intercal \tilde{f}_t(w_t)$
        \STATE Newton step: $w' =w_t - \frac{1}{\beta}H_t^{-1} \nabla \tilde{f}_t(w_t)$
        \STATE project: $w_{t+1} = \argmin_{w\in \Delta'_k} (w-w')^\intercal H_t (w-w')$
    \ENDFOR 
  \end{algorithmic}
\end{algorithm}

In the partial information setting, the natural performance measure of the player is the \emph{expected regret} $\expval{\regret_T}$, where the expectation is taken with respect to the randomized choices of the player and actions of the adversary. Crucially, we allow the adversary to adapt to the player's past behavior. This non-oblivious setting naturally arises in stochastic optimization, where $\ell_t$ depends on $w_{t-1}$. For analyzing the expected regret incurred by the VRM under partial information, we can reuse the full information analysis. However, the exp-concavity constant and the gradient norm bounds change, and the non-oblivious behavior requires further analysis, resulting in the no-regret guarantee of Theorem \ref{thm:main-partial-info}, which is \emph{independent} of $n$. 

\begin{theorem} \label{thm:main-partial-info}
VRM achieves the expected regret 
\begin{equation*}
\expval{\regret_T} = \tilde{\mathcal{O}} \left(  k^{3/8} c^{1/5}  L T^{4/5} \right).
\end{equation*}
\end{theorem}
\begin{proofarg}{sketch}
We first start by bounding the pseudo-regret,  which compares the cost
incurred by VRM to the cost incurred by the optimal mixture weights in expectation. It can be shown that $\tilde{f}_t(w)$ is $\frac{2\gamma^2}{n^2L}$-exp concave on $\Delta_k'$ and has the gradient bound
\begin{equation*}
\norm{\nabla \tilde{f}_t(w)} =  \frac{\tell_t^2(I_t) \norm{p(I_t)} }{(w^\intercal p(I_t))^2}  \leq \frac{L n^2 c \sqrt{k}}{\gamma^3},
\end{equation*}
where the inequality uses the fact that $w^\intercal p(I_t) \geq \gamma/n$  and Assumption 1,  which  implies $\norm{p(i)} \leq c \sqrt{k} / n$ for all $i \in [n]$. Combined with the guarantee in Equation \ref{eq:ons-bound}, this gives the bound on the expectation of $(A)$ from the regret decomposition. The upper bound on $(B)$ from Lemma \ref{lemma:b-bound} does not change under expectation and the modified losses. For bounding the expected regret,  we rely on Freedman's lemma \citep{freedman1975tail} for the martingale difference sequence $\{Z_{t}: =\sumin  \tell_t^2(i) - \sumin \ell_t^2(i) \}_{t\in[T]}$ in order to account for the non-oblivious nature of the adversary.
\end{proofarg}

\section{Efficient Implementation} We now address practical aspects of VRM. Naively implemented, each iteration  of the algorithm has a complexity of $\mathcal{O}(k^3)$. One might argue that this can become a bottleneck when performed in each round of stochastic optimization. In practice, however, one usually has a limited number of available proposal distributions, limiting $k$ to the small regime. Moreover, in the following, we present several tricks that improve on the complexity of the iteration.

The online Newton update and step in lines 5 and 6 of Algorithm \ref{alg:ons-bandit} can be implemented in $\mathcal{O}(k^2)$ due to the Sherman-Morrison formula \citep{hazan2006logarithmic}:
\begin{equation*}
H_t^{-1}= H_{t-1} ^{-1} - \frac{ H_{t-1}^{-1}\nabla \tilde{f}_t(w_t)\nabla \tilde{f}_t(w_t)^\intercal  H_{t-1}^{-1}}{1 + \nabla \tilde{f}_t(w_t)^\intercal H_{t-1}^{-1} \nabla \tilde{f}_t(w_t)}.
\end{equation*}
 Thus, the most costly operation of the algorithm is the projection step that requires solving a quadratic program, having a complexity of $\mathcal{O}(k^3)$. In practice,  we can trade off accuracy for efficiency in solving the quadratic program approximately by employing only a few steps  of a projection-based iterative solver (e.g., projected gradient descent, etc.). The key to the success of such a proposal is an efficient projection step onto the restricted simplex $\Delta'_k$, which captures the constraints of the quadratic program. Our proposed method, Algorithm \ref{alg:pgd}, is a two-stage projection procedure that is inspired by the efficient projection onto the simplex \citep{gafni1984two, duchi2008efficient} and has $\mathcal{O}(k \log k)$ time  complexity due to the sorting.

\begin{algorithm}[h]
  \caption{Projection}\label{alg:pgd}
  \begin{algorithmic}[1]
  \INPUT $w$, $\gamma$
  \FUNCTION{proj\_simplex $(w\in \mathbb{R}^d, z\in (0, 1])$}
  	\STATE sort $w$ decreasingly into $u$
  	\STATE $\rho = \max \left\{ j \in [d]: u_j -  \left( \sum_{\tau=1}^j u_\tau - z  \right) / j >0  \right\}$ 
  	\STATE $\lambda = \left( \sum_{\tau=1}^\rho u_\tau - z  \right ) / \rho$
 	\STATE \textbf{return} $\max \{w - \lambda, 0\}$
  \ENDFUNCTION
  \STATE 
    		\STATE $w=$ proj\_simplex $(w,1)$
    		\IF{$w(k) < \gamma$}
    			\STATE $w(k) = \gamma$
    			\STATE $w(1:k-1)=$ proj\_simplex $(w(1:k-1),1-\gamma)$
    		\ENDIF 
    \STATE \textbf{return} $w$
  \end{algorithmic}
\end{algorithm}

The idea behind projecting to $\Delta'_k$ is the following: if the projection step with respect to the full simplex results in a point in the restricted simplex, we are done. Otherwise, we set the last coordinate of $w$ to $\gamma$, and project the first $k-1$ coordinates to have mass $1-\gamma$.

\begin{figure*}[t]
  \centering
  \begin{subfigure}[c]{0.3\linewidth}
  	\vspace{-4mm}
    \includegraphics[width=\linewidth]{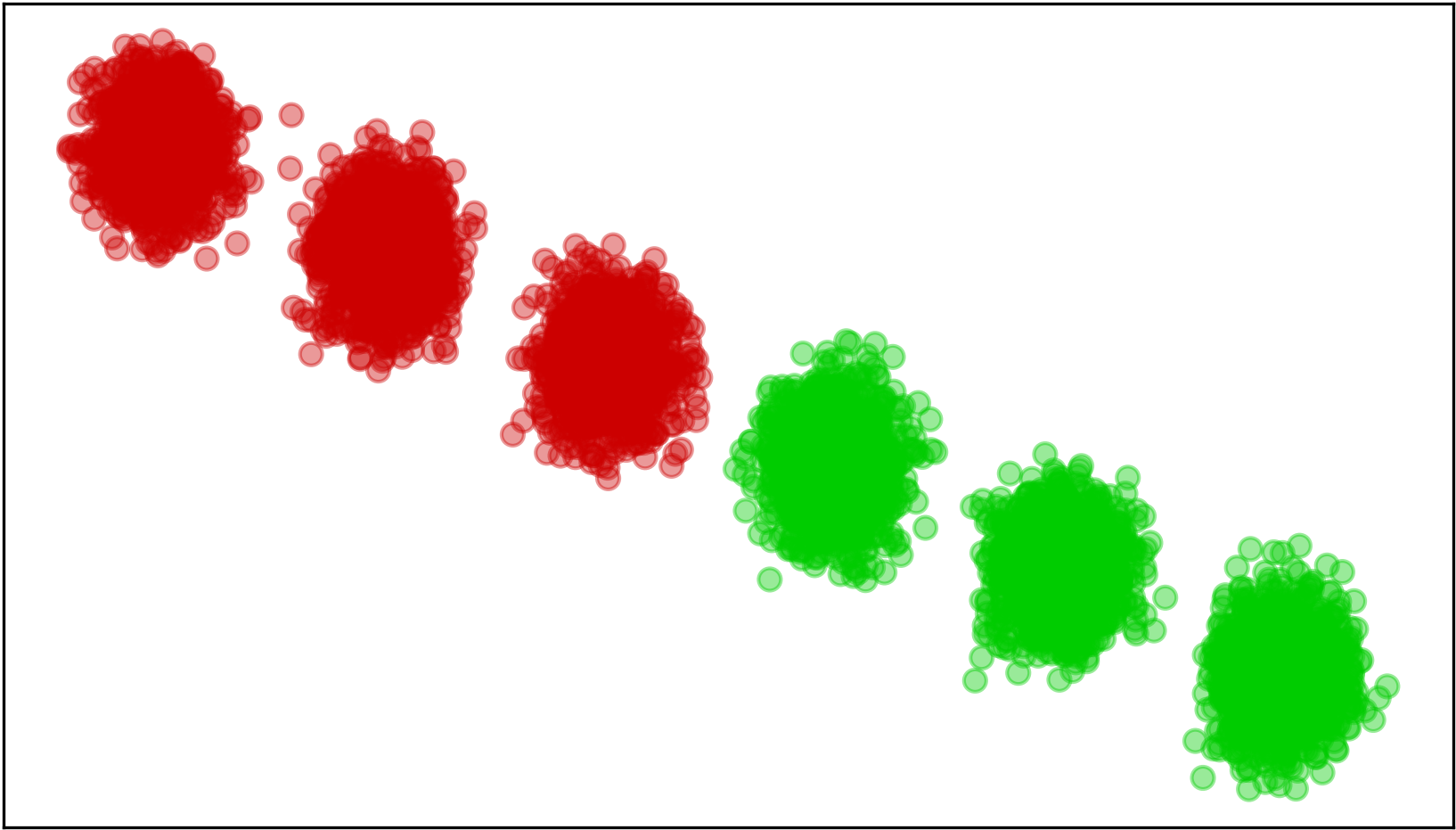}
  \end{subfigure}
  \hspace{2mm}
  \begin{subfigure}[c]{0.3\linewidth}
  \vspace{-4mm}
    \includegraphics[width=\linewidth]{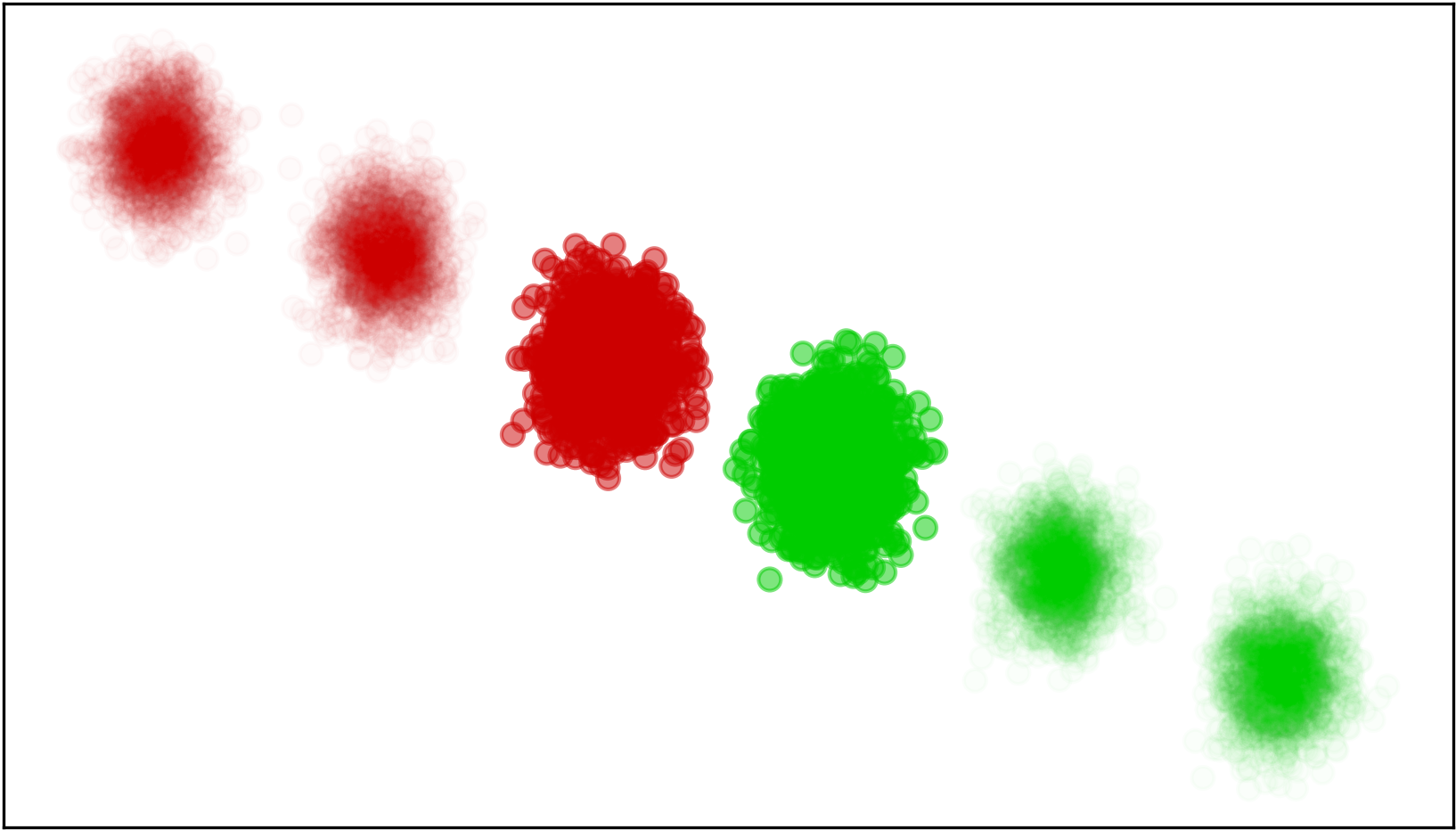}
  \end{subfigure}
   \hspace{2mm}
  \begin{subfigure}[c]{0.34\linewidth}
    \includegraphics[width=\linewidth]{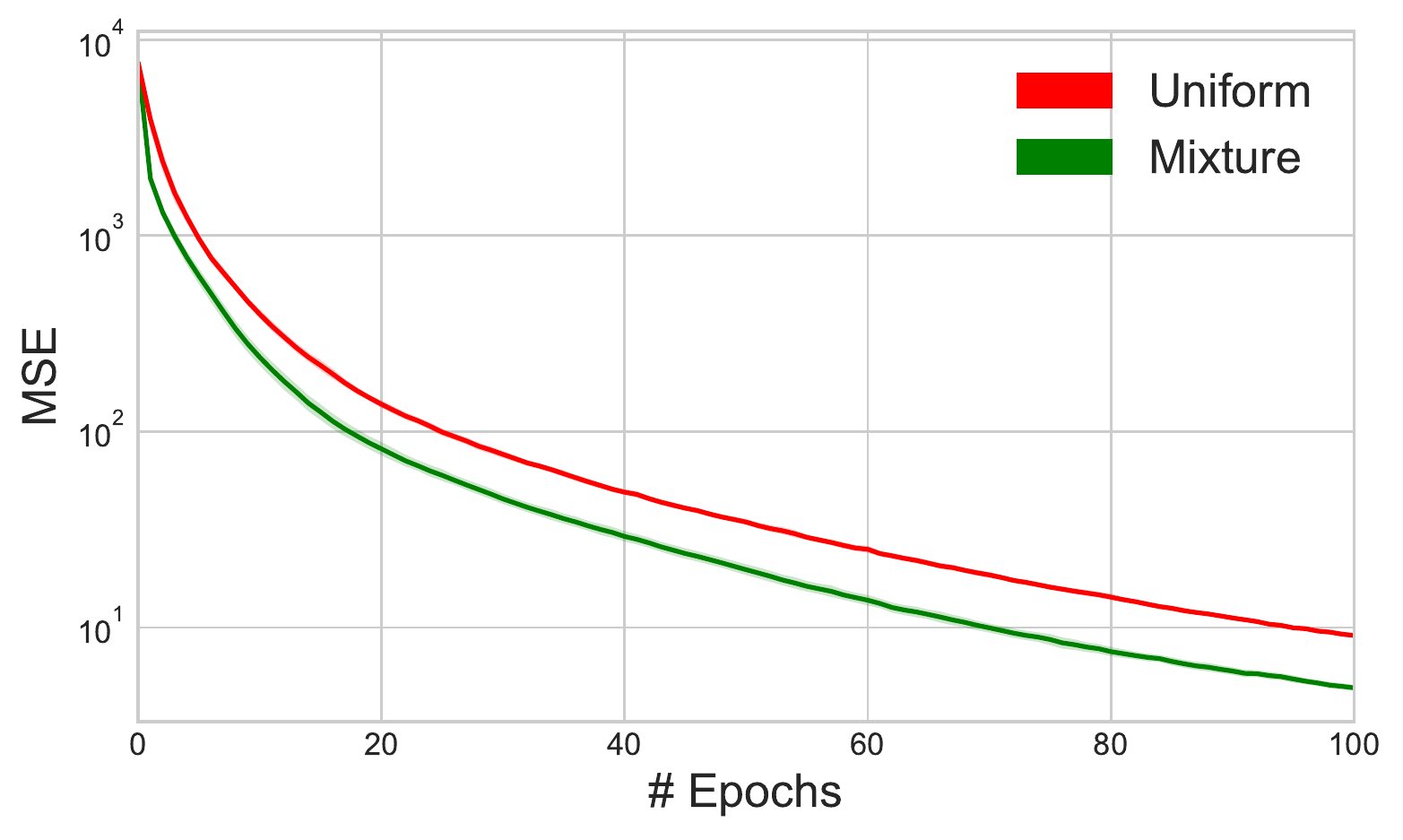}
  \end{subfigure}
  \vspace{-2mm}
 \caption{ Left: toy dataset consisting of 6 blobs, green indicates positive labels. Middle: illustration of mixture weights after $10\,000$ iterations, where high transparency corresponds to low weight. Due to large mixture weights, points from the two middle blobs are sampled more often, leading to faster discovery of support vectors. Right: Mean squared error achieved by the samplers on the regression task. VRM with $k$-DPPs  provides $1.4\times$ speedup over uniform sampling in terms of iterations.\ \label{fig:toy}}
 \vspace{-2mm}
\end{figure*}

\begin{lemma}\label{lemma:projection}
Algorithm  \ref{alg:pgd}  returns 
\begin{equation*}
x = \arg \min_{x'} {\norm{x'-w}^2} \quad \textup{s.t.}  \; x \in \Delta_k'.
\end{equation*}
\end{lemma}
\begin{proof}
As shown by \citet{duchi2008efficient}, the \emph{proj\_simplex} function solves the following minimization problem:
\begin{equation*}
\min_x {\norm{x-w}^2} \quad \textup{s.t.} \; \sum_{i=1}^d x_i = z,  \, x_i \geq 0.
\end{equation*}
Denoting $x_* = \arg \min_{x \in \Delta_k} \norm{x-w}$ and $x'_* = \arg \min_{x \in \Delta'_k} \norm{x-w}$, we only need to inspect the case when $x_* \neq x'_*$. In this case, we have $x'_*(k) = \gamma$. To see this by proof of contradiction, assume $x'_*(k) > \gamma$. Now we have $\|x_*-w\|< \|x'_*-w\|$\footnote{This is since the projection objective $\|x-w\|^2$ is  strongly-convex,  and hence the optimum must be unique.},
and there also exists a small $\epsilon$ such that $y:=(1-\epsilon)x'_* + \epsilon x_* \in \Delta'_k$ and $y(k) = \gamma$. The contradicts with the  optimality of  $x'_*$ since, 
\begin{equation*}
\norm{y-w} \leq (1-\epsilon)\norm{x'_* -w} +  \epsilon \norm{x_* - w} <  \norm{x'_*-w}.
\end{equation*}
As a consequence, if  $x_* \neq x'_*$ we can set  $w(k)=\gamma$ and  call the \emph{proj\_simplex} function for the first $k-1$ coordinates and with the $1-\gamma$ leftover mass.
\end{proof}
Thus we have reduced the cost of one iteration in VRM to $\mathcal{O} (k^2 )$,  and we further investigate its efficiency in the experiments.

\section{Experiments} \label{sec:experiments}
In this section, we evaluate our method experimentally. The  experiments are designed to illustrate the underlying principles of the algorithm as well as the beneficial effects of variance reduction in various real-world domains. We emphasize that it is crucial to design good sampling distributions for the mixture, and that this is an application-specific task. The following experiments provide guidance to this design process, but deriving better sampling distributions is an open question for future work.

\subsection{SVM on blobs}

Consider the toy dataset consisting of $n=10\,000$ datapoints arranged in 6 well-separated, balanced, Gaussian blobs illustrated in the left of Figure \ref{fig:toy}. Points belonging to the leftmost three blobs are assigned negative class labels, and points in the rightmost three are labelled as positive.

In this setting it is natural to propose $k=6$ sampling distributions, one corresponding to each blob. A specific component assigns uniformly large probability to its associated points and uniformly small probability everywhere else. Notice that in this case $c=k$. We run 5 epochs of online gradient descent for SVM with step size $0.01/\sqrt{t}$ at iteration $t$. At each iteration, the sampler gets as feedback the norm of the gradient of the hinge loss. 
This way, VRM is expected to propose critical points (producing high norm loss gradients) more frequently, i.e,. to sample the two middle blobs often, since they contain the support vectors. This intuition is confirmed in the middle plot of Figure \ref{fig:toy}, where the points' color intensities represent their corresponding blob's mixture weights obtained by VRM at the end of the training. This also results in the fact that VRM achieves a certain level of accuracy faster than uniform sampling, due to discovering the support vectors earlier.

\subsection{$k$-DPPs}
The following experiment illustrates that our method can handle distributions over sets of points. $k$-Determinantal point processes ($k$-DPP) \citep{kulesza2012determinantal} over a discrete set is a distribution over all subsets of size $k$. Being a member of the family of repulsive point processes, their diversity-inducing effect has recently been used in \citet{zhang2017determinantal} for sampling minibatches in stochastic optimization. In this experiment, we take a similar path  and investigate variance reduction in linear regression with sampling batches from a mixture of $k$-DPP kernels. This is rendered possible by our theoretical results, which show that the regret is independent of the number of points (which is ${n\choose k}$ in this case).

We solve linear regression on a synthetic dataset of size $n=1\,000$ and dimension $d=10$ generated as follows: the features are drawn from a multivariate normal distribution with random means and variances for each dimension. In order to change the relative importance of the samples, the features of 10 randomly selected points are scaled up by a factor of 10. The dependent variables $Y$ are generated by $Y= Xw_0+\epsilon$, where $X$ is the feature matrix, $w_0$ is a vector drawn from a normal distribution with mean 0 and variance 25 and $\epsilon$ is the standard normal noise.
The optimization is performed with minibatch SGD with step size $10^{-4}/\sqrt{t}$ in round $t$ over 100 epochs and batch size of 5. The feedback to the samplers is norm of the gradient of the mean squared error.

Our mixture consists of three $k$-DPPs with regularized linear kernel $L=X X^\intercal + \lambda \mathbb{I} $, where $\lambda \in \{1, 10, 100\}$. We introduce a small bias  by applying soft truncation to the importance weights: $r' = 0.8r + 0.2$.
The result of the 10 runs of the optimization process with different random seeds  shown in right of Figure \ref{fig:toy}, where VRM significantly outperforms the uniform sampling in terms of number of iterations needed for a certain error level.  However, since we use exact $k$-DPP sampling, the computational overhead outweighs the practical benefits of our method in this setting\footnote{Efficient $k$-DPP samplers are available, e.g. \citep{li2016efficient}; we leave the investigation of time-performance trade-offs with these samplers for future work.}.  

\begin{figure*}[t]
  \centering
    \includegraphics[width=\linewidth]{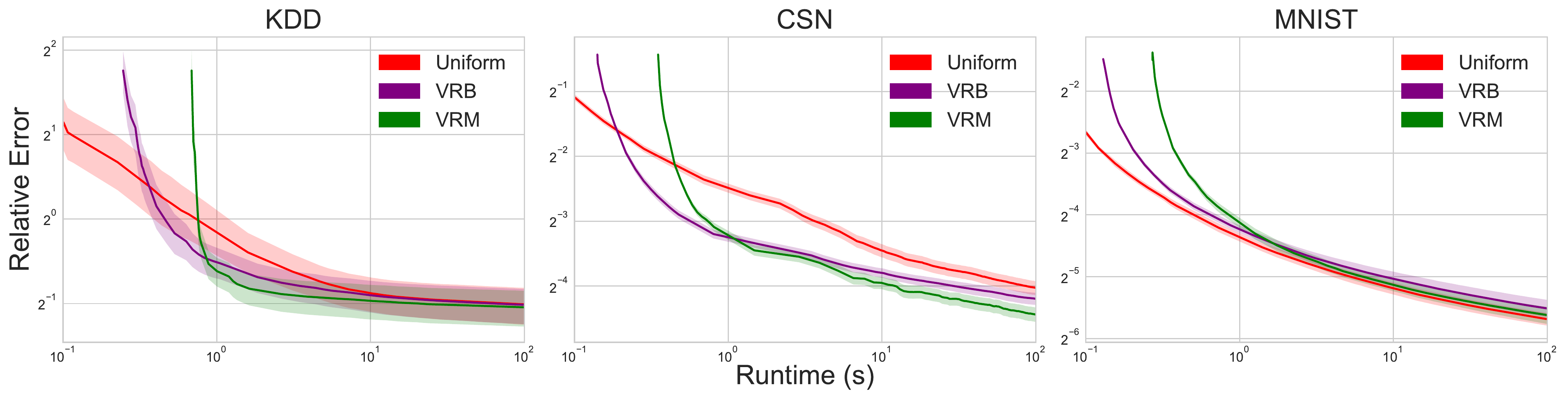}
        \vspace{-5mm}
 \caption{ $k$-means loss evolution on the test set. VRM suffers from a larger setup time due the cost of initializing the mixture components, but eventually outperforms the other methods in terms of relative error, where the reference is the batch $k$-means. \label{fig:kmeans}}
 \vspace{-2mm}
\end{figure*}

\subsection{Prioritized Experience Replay}

In this experiment, our goal is to identify good hyperparameters for prioritized experience replay \citep{Schaul2016} with  Deep Q-Learning (DQN) \citep{mnih2015human} on the Cartpole environment of the Gym  \citep{1606.01540}.
Prioritized experience replay is an importance sampling scheme that samples observations from the replay buffer approximately proportional to their last associated temporal difference (TD) error. The sampling distribution over a point $j$ in the buffer is $p(j) \propto (|\delta_j | + \epsilon)^\alpha$, where $\delta_j $ is the last observed TD-error associated to experience $j$, whereas $\epsilon$ and $\alpha$ are hyperparameters for smoothing the probabilities. With the appropriately chosen hyperparameters, prioritized experience replay can significantly improve the performance of DQN learning.
\begin{figure}[H]
  \centering
    \includegraphics[width=0.9\linewidth]{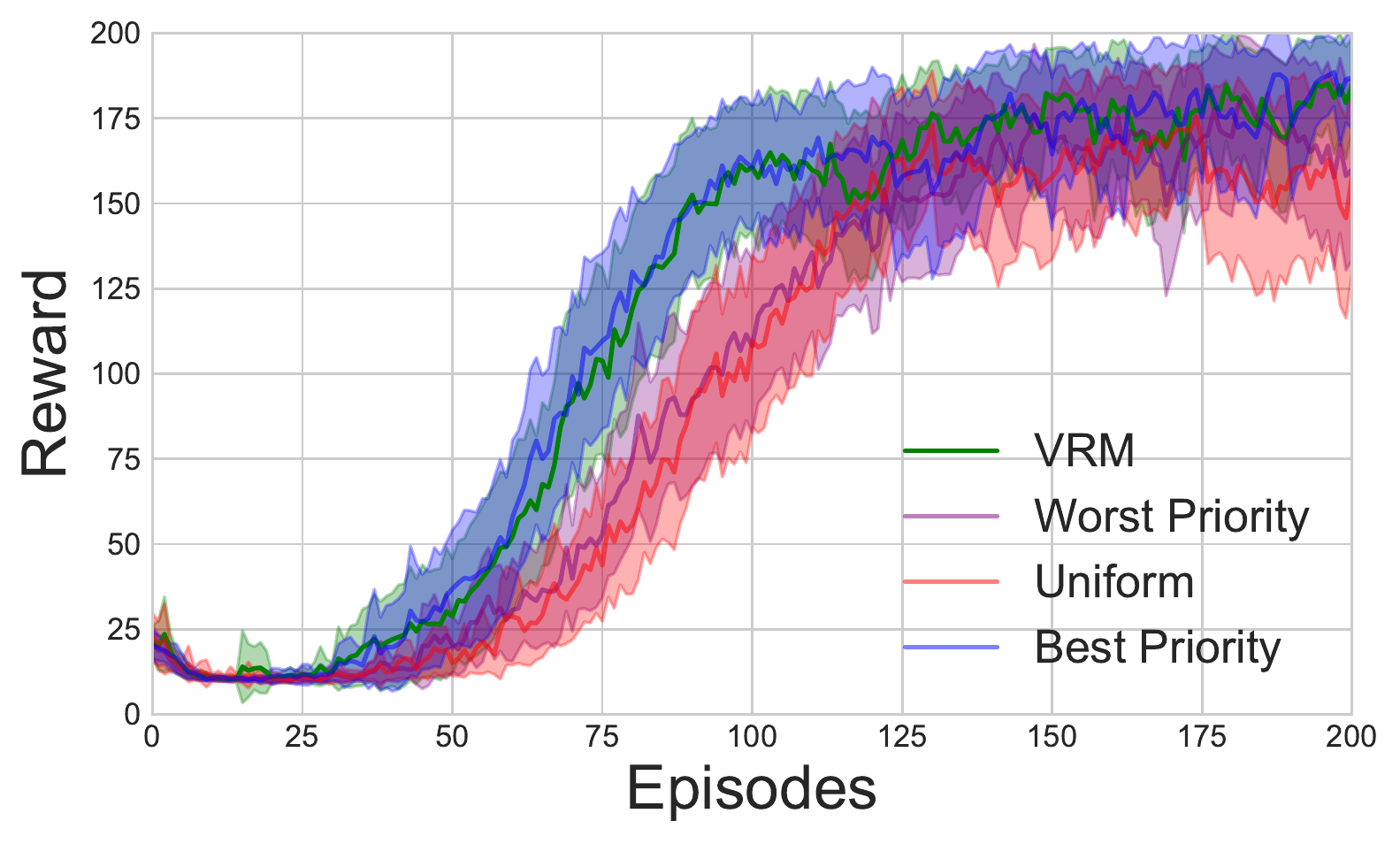}
    \vspace{-2mm}
 \caption{ Evolution of rewards over 200 episodes of the different experience replay samplers on Cartpole.  50 runs with different random seeds. VRM identifies the mixture component corresponding to the best hyperparameter setting in early stages and assigns a large mixture weight to it. \label{fig:cartpole}}
 \vspace{-2mm}
\end{figure}
     
In this experiment, we show how VRM allows for automatic hyperparameter selection in a single run without loss in the performance. We generate 9 mixture components of prioritized experience replays with  all the possible parameter combinations of $\epsilon = \{ 0.01, 0.1, 1\}$ and $\alpha=\{ 0.1, 0.5, 0.9\}$.
The feedback to the VRM is the TD-error incurred by the sampled experiences. During the optimization process, the prioritized replay buffers are also updated as new observations are inserted and the TD-errors are refreshed. This is a deviation from our presentation, where we relied on fixed sampling distributions. However, it is straightforward to see that our framework easily extends to sampling distributions \emph{changing} over time, i.e., sampling point $i$ in round $t$ is $i \sim w_t^\intercal p_t(i)$ and we allow $ p_t$ to depend on $t$. The result of 50 runs with different random seeds over 200 episodes is presented in Figure \ref{fig:cartpole}.  VRM successfully identifies the mixture component corresponding to the best hyperparameter setting in early stages and assigns the largest mixture weight to it. As a consequence, VRM performs hyperparameter selection in a single run without loss of performance compared to the best setting.

\subsection{$k$-means}

Next we investigate the gains of our sampler for minibatch $k$-means \citep{sculley2010web}. We reproduce the  experimental setup  of  \citet{pmlr-v75-borsos18a}, and  compare VRM to uniform sampling and to VRB. The parameters of VRB where chosen as indicated by the authors. For both VRM and VRB, the points in the batch are sampled $\emph{independently}$ according to the samplers and the feedback is given in a delayed fashion, once per batch. The feedback corresponds to the norm of the minibatch $k$-means loss gradient. 

It remains to specify how to construct our mixture sampler.  We use a mixture with 10 components. Inspired by VRB, we choose each mixture's sampling distribution proportional to the square root of the distances to a randomly chosen center  with small uniform smoothing. More formally, for each component $j$, we define its sampling distribution as 
\begin{equation*}
p_j(i) = \frac{0.9 \cdot \sqrt{d^2(x_i, \mu_j)}}{\sqrt{\sum_{k=1}^n d^2(x_k, \mu_j)}} +  \frac{0.1}{n},
\end{equation*}
where $\mu_j$ is the randomly chosen center for component $j$. We note that this design of sampling distributions leads to low values of $c$, as presented in Table \ref{table:datasets}.

We use batch size $b=100$ and number of clusters $k=100$, and initialize the centers via $k$-means++ \citep{arthur2007k}, where the initialization is shared across all methods. We generate 10 different set of initial centers and run each version 10 times on each set of initial centers. We train the algorithms on $80\%$ of the data. For the mixture sampler, we perform an additional $80\%$-$20\%$ split the training data, in order to choose the hyperparameters $\beta$ and $\gamma$. We report the loss on the remaining $20\%$ test set on the datasets presented in Table \ref{table:datasets} \citep{ kddcup2004,faulkner2011next, lecun1998gradient} with more details in the supplementary material.
\begin{table}[h]
\caption{Dataset details \label{table:datasets}}
\vskip 0.1in
\centering
\begin{tabular}{@{}cccc@{}}
\toprule
&\textbf{KDD}& \textbf{CSN} & \textbf{MNIST} \\ \midrule
nr. of points & $145\,751$       & $80\,000$        & $70\,000$         \\ \bottomrule
nr. of features & 74      & 17       & 10           \\ \bottomrule
 c  & 48.18     & 42.42     & 3.09           \\ \bottomrule
\end{tabular}
\end{table}

We are ultimately interested in the performance versus computational time trade-off.  Thus, for the samplers, we include in the time measurement the setup and the sampling time. The results are shown in Figure \ref{fig:kmeans}, where we measure the relative error of minibatch $k$-means combined with different samplers compared to batch $k$-means. The shaded areas represent $95\%$ confidence intervals.  VRM suffers initially from a high setup time due to calculation of the proposed sampling distributions of the mixture, but eventually outperforms the other methods. Similarly to \citet{pmlr-v75-borsos18a},  we observe no advantage on MNIST, where the best-in-hindsight mixture weights are uniform.

\section{Conclusion}
We proposed a novel framework for online variance reduction with mixtures, in  which structures in the data can be easily captured by formulating fixed sampling distributions as mixture components. We devised VRM, a novel importance sampling method for this setting that relies on the Online Newton Step algorithm and showed that it asymptotically recovers the optimal mixture weight in hindsight. After several considerations for improving efficiency, including a novel projection step on the restricted simplex, we empirically demonstrate the versatility of VRM in a range of applications.

\section*{Acknowledgements}
 This research was supported by SNSF grant $407540\_167212$ through the NRP 75 Big Data program.
K.Y.L. is supported by the ETH Zurich Postdoctoral Fellowship and Marie Curie Actions for People COFUND program.

\newpage
\bibliographystyle{icml2019}
\bibliography{bibliography}

\begin{thebibliography}{36}
\providecommand{\natexlab}[1]{#1}
\providecommand{\url}[1]{\texttt{#1}}
\expandafter\ifx\csname urlstyle\endcsname\relax
  \providecommand{\doi}[1]{doi: #1}\else
  \providecommand{\doi}{doi: \begingroup \urlstyle{rm}\Url}\fi

\bibitem[Allen-Zhu et~al.(2016)Allen-Zhu, Qu, Richt{\'a}rik, and
  Yuan]{allen2016even}
Allen-Zhu, Z., Qu, Z., Richt{\'a}rik, P., and Yuan, Y.
\newblock Even faster accelerated coordinate descent using non-uniform
  sampling.
\newblock In \emph{International Conference on Machine Learning}, pp.\
  1110--1119, 2016.

\bibitem[Arthur \& Vassilvitskii(2007)Arthur and Vassilvitskii]{arthur2007k}
Arthur, D. and Vassilvitskii, S.
\newblock k-means++: The advantages of careful seeding.
\newblock In \emph{Proceedings of the eighteenth annual ACM-SIAM symposium on
  Discrete algorithms}, pp.\  1027--1035. Society for Industrial and Applied
  Mathematics, 2007.

\bibitem[Auer et~al.(2002)Auer, Cesa-Bianchi, Freund, and
  Schapire]{auer2002nonstochastic}
Auer, P., Cesa-Bianchi, N., Freund, Y., and Schapire, R.~E.
\newblock The nonstochastic multiarmed bandit problem.
\newblock \emph{SIAM journal on computing}, 32\penalty0 (1):\penalty0 48--77,
  2002.

\bibitem[Borsos et~al.(2018)Borsos, Krause, and Levy]{pmlr-v75-borsos18a}
Borsos, Z., Krause, A., and Levy, K.~Y.
\newblock Online variance reduction for stochastic optimization.
\newblock In Bubeck, S., Perchet, V., and Rigollet, P. (eds.),
  \emph{Proceedings of the 31st Conference On Learning Theory}, volume~75 of
  \emph{Proceedings of Machine Learning Research}, pp.\  324--357. PMLR, 06--09
  Jul 2018.

\bibitem[Bouchard et~al.(2015)Bouchard, Trouillon, Perez, and
  Gaidon]{bouchard2015online}
Bouchard, G., Trouillon, T., Perez, J., and Gaidon, A.
\newblock Online learning to sample.
\newblock \emph{arXiv preprint arXiv:1506.09016}, 2015.

\bibitem[Brockman et~al.(2016)Brockman, Cheung, Pettersson, Schneider,
  Schulman, Tang, and Zaremba]{1606.01540}
Brockman, G., Cheung, V., Pettersson, L., Schneider, J., Schulman, J., Tang,
  J., and Zaremba, W.
\newblock Openai gym, 2016.

\bibitem[Defazio et~al.(2014)Defazio, Bach, and
  Lacoste-Julien]{defazio2014saga}
Defazio, A., Bach, F., and Lacoste-Julien, S.
\newblock Saga: A fast incremental gradient method with support for
  non-strongly convex composite objectives.
\newblock In \emph{Advances in Neural Information Processing Systems}, pp.\
  1646--1654, 2014.

\bibitem[Duchi et~al.(2008)Duchi, Shalev-Shwartz, Singer, and
  Chandra]{duchi2008efficient}
Duchi, J., Shalev-Shwartz, S., Singer, Y., and Chandra, T.
\newblock Efficient projections onto the l 1-ball for learning in high
  dimensions.
\newblock In \emph{Proceedings of the 25th international conference on Machine
  learning}, pp.\  272--279. ACM, 2008.

\bibitem[Duchi et~al.(2011)Duchi, Hazan, and Singer]{duchi2011adaptive}
Duchi, J., Hazan, E., and Singer, Y.
\newblock Adaptive subgradient methods for online learning and stochastic
  optimization.
\newblock \emph{Journal of Machine Learning Research}, 12\penalty0
  (Jul):\penalty0 2121--2159, 2011.

\bibitem[Faulkner et~al.(2011)Faulkner, Olson, Chandy, Krause, Chandy, and
  Krause]{faulkner2011next}
Faulkner, M., Olson, M., Chandy, R., Krause, J., Chandy, K.~M., and Krause, A.
\newblock The next big one: Detecting earthquakes and other rare events from
  community-based sensors.
\newblock In \emph{Information Processing in Sensor Networks (IPSN), 2011 10th
  International Conference on}, pp.\  13--24. IEEE, 2011.

\bibitem[Freedman(1975)]{freedman1975tail}
Freedman, D.~A.
\newblock On tail probabilities for martingales.
\newblock \emph{the Annals of Probability}, pp.\  100--118, 1975.

\bibitem[Gafni \& Bertsekas(1984)Gafni and Bertsekas]{gafni1984two}
Gafni, E.~M. and Bertsekas, D.~P.
\newblock Two-metric projection methods for constrained optimization.
\newblock \emph{SIAM Journal on Control and Optimization}, 22\penalty0
  (6):\penalty0 936--964, 1984.

\bibitem[Hazan et~al.(2006)Hazan, Kalai, Kale, and
  Agarwal]{hazan2006logarithmic}
Hazan, E., Kalai, A., Kale, S., and Agarwal, A.
\newblock Logarithmic regret algorithms for online convex optimization.
\newblock In \emph{Lecture Notes in Computer Science}, volume 4005, pp.\
  499--513. Springer-Verlag Berlin Heidelberg, June 2006.

\bibitem[Hazan et~al.(2016)]{hazan2016introduction}
Hazan, E. et~al.
\newblock Introduction to online convex optimization.
\newblock \emph{Foundations and Trends{\textregistered} in Optimization},
  2\penalty0 (3-4):\penalty0 157--325, 2016.

\bibitem[Johnson \& Zhang(2013)Johnson and Zhang]{johnson2013accelerating}
Johnson, R. and Zhang, T.
\newblock Accelerating stochastic gradient descent using predictive variance
  reduction.
\newblock In \emph{Advances in neural information processing systems}, pp.\
  315--323, 2013.

\bibitem[Johnson \& Guestrin(2018)Johnson and Guestrin]{NIPS2018_7957}
Johnson, T.~B. and Guestrin, C.
\newblock Training deep models faster with robust, approximate importance
  sampling.
\newblock In Bengio, S., Wallach, H., Larochelle, H., Grauman, K.,
  Cesa-Bianchi, N., and Garnett, R. (eds.), \emph{Advances in Neural
  Information Processing Systems 31}, pp.\  7276--7286. Curran Associates,
  Inc., 2018.

\bibitem[Kakade \& Tewari(2009)Kakade and Tewari]{kakade2009generalization}
Kakade, S.~M. and Tewari, A.
\newblock On the generalization ability of online strongly convex programming
  algorithms.
\newblock In \emph{Advances in Neural Information Processing Systems}, pp.\
  801--808, 2009.

\bibitem[Katharopoulos \& Fleuret(2018)Katharopoulos and
  Fleuret]{pmlr-v80-katharopoulos18a}
Katharopoulos, A. and Fleuret, F.
\newblock Not all samples are created equal: Deep learning with importance
  sampling.
\newblock In Dy, J. and Krause, A. (eds.), \emph{Proceedings of the 35th
  International Conference on Machine Learning}, volume~80 of \emph{Proceedings
  of Machine Learning Research}, pp.\  2525--2534, Stockholmsmässan, Stockholm
  Sweden, 10--15 Jul 2018. PMLR.
\newblock URL \url{http://proceedings.mlr.press/v80/katharopoulos18a.html}.

\bibitem[KDD Cup 2004()]{kddcup2004}
KDD Cup 2004.
\newblock {KDD Cup 2004. Protein Homology Dataset.}
\newblock \url{http://osmot.cs.cornell.edu/kddcup/}, 2004.
\newblock Accessed: 10.11.2016.

\bibitem[Kulesza et~al.(2012)Kulesza, Taskar, et~al.]{kulesza2012determinantal}
Kulesza, A., Taskar, B., et~al.
\newblock Determinantal point processes for machine learning.
\newblock \emph{Foundations and Trends{\textregistered} in Machine Learning},
  5\penalty0 (2--3):\penalty0 123--286, 2012.

\bibitem[LeCun et~al.(1998)LeCun, Bottou, Bengio, and
  Haffner]{lecun1998gradient}
LeCun, Y., Bottou, L., Bengio, Y., and Haffner, P.
\newblock {Gradient-based learning applied to document recognition}.
\newblock \emph{Proceedings of the IEEE}, 86\penalty0 (11):\penalty0
  2278--2324, 1998.

\bibitem[Li et~al.(2016)Li, Jegelka, and Sra]{li2016efficient}
Li, C., Jegelka, S., and Sra, S.
\newblock Efficient sampling for k-determinantal point processes.
\newblock In \emph{Artificial Intelligence and Statistics}, pp.\  1328--1337,
  2016.

\bibitem[Loshchilov \& Hutter(2015)Loshchilov and Hutter]{loshchilov2015online}
Loshchilov, I. and Hutter, F.
\newblock Online batch selection for faster training of neural networks.
\newblock \emph{arXiv preprint arXiv:1511.06343}, 2015.

\bibitem[Mnih et~al.(2015)Mnih, Kavukcuoglu, Silver, Rusu, Veness, Bellemare,
  Graves, Riedmiller, Fidjeland, Ostrovski, et~al.]{mnih2015human}
Mnih, V., Kavukcuoglu, K., Silver, D., Rusu, A.~A., Veness, J., Bellemare,
  M.~G., Graves, A., Riedmiller, M., Fidjeland, A.~K., Ostrovski, G., et~al.
\newblock Human-level control through deep reinforcement learning.
\newblock \emph{Nature}, 518\penalty0 (7540):\penalty0 529, 2015.

\bibitem[Namkoong et~al.(2017)Namkoong, Sinha, Yadlowsky, and
  Duchi]{pmlr-v70-namkoong17a}
Namkoong, H., Sinha, A., Yadlowsky, S., and Duchi, J.~C.
\newblock Adaptive sampling probabilities for non-smooth optimization.
\newblock In \emph{Proceedings of the 34th International Conference on Machine
  Learning}, volume~70 of \emph{Proceedings of Machine Learning Research}, pp.\
   2574--2583, International Convention Centre, Sydney, Australia, 06--11 Aug
  2017. PMLR.

\bibitem[Needell et~al.(2014)Needell, Ward, and Srebro]{needell2014stochastic}
Needell, D., Ward, R., and Srebro, N.
\newblock Stochastic gradient descent, weighted sampling, and the randomized
  kaczmarz algorithm.
\newblock In \emph{Advances in Neural Information Processing Systems}, pp.\
  1017--1025, 2014.

\bibitem[Perekrestenko et~al.(2017)Perekrestenko, Cevher, and
  Jaggi]{perekrestenko2017faster}
Perekrestenko, D., Cevher, V., and Jaggi, M.
\newblock Faster coordinate descent via adaptive importance sampling.
\newblock In \emph{Proceedings of the 20th International Conference on
  Artificial Intelligence and Statistics}, volume~54. PMLR, 2017.

\bibitem[{Salehi} et~al.(2017){Salehi}, {Celis}, and {Thiran}]{salehi2017}
{Salehi}, F., {Celis}, L.~E., and {Thiran}, P.
\newblock {Stochastic Optimization with Bandit Sampling}.
\newblock \emph{ArXiv e-prints}, August 2017.

\bibitem[Salehi et~al.(2018)Salehi, Thiran, and Celis]{NIPS2018_8137}
Salehi, F., Thiran, P., and Celis, E.
\newblock Coordinate descent with bandit sampling.
\newblock In Bengio, S., Wallach, H., Larochelle, H., Grauman, K.,
  Cesa-Bianchi, N., and Garnett, R. (eds.), \emph{Advances in Neural
  Information Processing Systems 31}, pp.\  9267--9277. Curran Associates,
  Inc., 2018.

\bibitem[Schaul et~al.(2016)Schaul, Quan, Antonoglou, and Silver]{Schaul2016}
Schaul, T., Quan, J., Antonoglou, I., and Silver, D.
\newblock Prioritized experience replay.
\newblock In \emph{International Conference on Learning Representations},
  Puerto Rico, 2016.

\bibitem[Sculley(2010)]{sculley2010web}
Sculley, D.
\newblock Web-scale k-means clustering.
\newblock In \emph{Proceedings of the 19th international conference on World
  wide web}, pp.\  1177--1178. ACM, 2010.

\bibitem[Stich et~al.(2017)Stich, Raj, and Jaggi]{NIPS2017_7025}
Stich, S.~U., Raj, A., and Jaggi, M.
\newblock Safe adaptive importance sampling.
\newblock In \emph{Advances in Neural Information Processing Systems 30}, pp.\
  4384--4394. Curran Associates, Inc., 2017.

\bibitem[Zhang et~al.(2017)Zhang, Kjellstrom, and
  Mandt]{zhang2017determinantal}
Zhang, C., Kjellstrom, H., and Mandt, S.
\newblock Determinantal point processes for mini-batch diversification.
\newblock \emph{Conference on Uncertainty in Artificial Intelligence (UAI)},
  2017.

\bibitem[Zhao \& Zhang(2014)Zhao and Zhang]{zhao2014accelerating}
Zhao, P. and Zhang, T.
\newblock Accelerating minibatch stochastic gradient descent using stratified
  sampling.
\newblock \emph{arXiv preprint arXiv:1405.3080}, 2014.

\bibitem[Zhao \& Zhang(2015)Zhao and Zhang]{zhao2015stochastic}
Zhao, P. and Zhang, T.
\newblock Stochastic optimization with importance sampling for regularized loss
  minimization.
\newblock In \emph{Proceedings of the 32nd International Conference on Machine
  Learning (ICML-15)}, pp.\  1--9, 2015.

\bibitem[Zinkevich(2003)]{zinkevich2003online}
Zinkevich, M.
\newblock Online convex programming and generalized infinitesimal gradient
  ascent.
\newblock In \emph{Proceedings of the 20th International Conference on Machine
  Learning (ICML-03)}, pp.\  928--936, 2003.

\end{thebibliography}

\appendix

\onecolumn

\begin{section}{Online Gradient Descent (OGD) and the Full Information Setting}
Let us inspect the regret incurred by OGD for $(A)$ in the full information setting. Denote by $D = \max_{w_1, w_2 \in \Delta'_k} \norm{w_1-w_2}$ the diameter of the restricted simplex $\Delta'_k$. We clearly have $D=\sqrt{2}$. Furthermore, define the gradient norm bound $G$ as $\sup_{w\in \Delta'_k,\, t\in [T]} \norm{\nabla f_t(w)}$. In the full information setting, we have
\begin{equation*}
\norm{\nabla f_t (w)}= \norm {\sumin \frac{\ell_t^2(i)p(i)}{(w^\intercal p(i))^2}} \leq \frac{n^2L}{\gamma^2} \norm{\sumin p(i)} \\
= \frac{n^2L}{\gamma^2} \norm{(1,...,1)} =\frac{n^2L\sqrt{k}}{\gamma^2}=:G,
\end{equation*}
where the inequality uses that in $\Delta'_k$ we have $w^\intercal p(i) \geq \gamma/n$. 
\citet{zinkevich2003online} showed that the regret incurred by OGD is $\mathcal{O} (GD\sqrt{T})$.  Using $\mathcal{O} (GD\sqrt{T})$ as a bound for $(A)$ and the result from Lemma \ref{lemma:b-bound} for $(B)$, we can optimize over $\gamma$ to get the OGD full information regret of $\mathcal{O} (Lk^{1/6}T^{5/6})$, which is clearly weaker than the regret incurred by ONS. A similar argument also holds for the partial information setting.
\end{section}

\begin{section}{Full Information Setting Proofs}
Let us look at the properties of our cost function after restricting the simplex, i.e., $w \in \Delta_k'$, thus $\Delta_k'$ is bounded convex compact set with diameter $D = \sqrt{2}$:
\begin{align}
f_t(w) &= \sumin \frac{\ell_t^2(i)}{w^\intercal p(i)} \\
\nabla f_t(w) &=  - \sumin \frac{\ell_t^2(i) p(i) }{(w^\intercal p(i))^2} \\ 
\nabla^2 f_t(w) &=   2 \sumin \frac{\ell_t^2(i) p(i) p(i)^\intercal }{(w^\intercal p(i))^3}  \label{eq:hessian}
\end{align}
Let us look at the exp-concavity of our cost function, for which we have the following result:

\begin{lemma}{\cite{hazan2016introduction}} \label{lemma:exp-concavity} A twice differentiable function $g :\Delta_k' \mapsto \mathbb{R}$ is $\alpha$-exp concave iff for any $x \in \Delta_k' $:
\begin{equation*}
\nabla^2g(x) \succeq  \alpha \nabla g(x)  \nabla g(x) ^\intercal
\end{equation*}
\end{lemma}
\begin{proof}
By definition, $g(x)$ is $\alpha$-exp-concave iff $-e^{-\alpha g(x)}$ is convex. The gradient of $-e^{-\alpha g(x)}$  is $\alpha \cdot e^{-\alpha g(x)} \nabla g(x)$ and its Hessian is
\begin{equation*}
\nabla^2(-e^{-\alpha g(x)}) = \nabla( \alpha \cdot e^{-\alpha g(x)} \nabla g(x) ) =  \alpha \cdot e^{-\alpha g(x)} \left( \nabla^2 g(x) - \alpha   \nabla g(x)  \nabla^\intercal g(x) \right).
\end{equation*}
Since a twice differentiable function on $\Delta_k'$ is convex iff its Hessian is PSD, and since $ \alpha \cdot e^{-\alpha g(x)} > 0$, we have our desired result.
\end{proof}
We are now ready to prove Lemma \ref{lemma:exp-concave}.
\begin{proofarg}{of Lemma \ref{lemma:exp-concave}}
For the proof, we start with a simple observation: let $x_1,...,x_n$ be vectors in $\mathbb{R}^d$, then:
\begin{equation} \label{eq:matrix-jensen-app} 
\left(\sumin x_i \right) \left(\sumin x_i \right)^\intercal  \preceq  n \sumin x_i x_i^\intercal,
\end{equation}
where $A \succeq B$ iff $A-B$ is PSD. To see this, we use the definition of positive semi-definiteness, $A \succeq B$ iff $u^\intercal (A-B) u\geq 0$ for all $u \in \mathbb{R}^d$. Using this for Eq. \ref{eq:matrix-jensen-app} and denoting $y_i := x_i^\intercal u \in \mathbb{R}$, we have:
\begin{equation*}
u^\intercal \left(n \sumin x_i x_i^\intercal -  \left(\sumin x_i \right) \left(\sumin x_i \right)^\intercal \right) u = n \sumin y_i^2 - \left(\sumin y_i \right)^2 \geq 0,
\end{equation*}
where the last inequality uses Jensen's inequality. Using this observation, we can now proceed,
\begin{align*}
\nabla f_t(w) \nabla f_t(w)^\intercal  &=  \left( \sumin \frac{\ell_t^2(i) p(i) }{(w^\intercal p(i))^2}  \right)  \left( \sumin \frac{\ell_t^2(i) p(i) }{(w^\intercal p(i))^2}  \right)^\intercal  \overset{(\ref{eq:matrix-jensen-app})}{ \preceq } n \sumin \frac{\ell_t^4(i) p(i) p(i)^\intercal }{(w^\intercal p(i))^4}  \\
&  \preceq  nL \sumin \frac{\ell_t^2(i) p(i) p(i)^\intercal }{(w^\intercal p(i))^4}  \preceq  \frac{n^2L}{2\gamma} \sumin \frac{2\ell_t^2(i) p(i) p(i)^\intercal }{(w^\intercal p(i))^3},
\end{align*}
where the last inequality uses the fact that $w^\intercal p(i) \geq \gamma / n$ since $w \in \Delta'_k$. However, on the RHS we can recognize the Hessian from Eq. \ref{eq:hessian}. Thus, identifying $\alpha = \frac{2\gamma}{n^2 L}$ in Lemma \ref{lemma:exp-concavity}, we finished the proof.
\end{proofarg}

\begin{proofarg}{of Theorem \ref{thm:main}}
From Lemmas \ref{lemma:b-bound} and \ref{lemma:a-bound} we have
\begin{equation}
\regret_T \leq \frac{10L k^{3/2} \log T}{\gamma^2} + \gamma LT.
\end{equation}
 We can optimize over $\gamma$ and set it to $\gamma = 3 k^{1/2} T^{-1/3}\log^{1/3}T$ in order to get the result.
\end{proofarg}

\end{section}

\begin{section}{Partial Information Setting Proofs}
Let us inspect the cost function estimate's properties in partial information setting:
\begin{align}
\tilde{f}_t(w) &= \frac{\tell_t^2(I_t)}{w^\intercal p(I_t)} \\
\nabla \tilde{f}_t(w) &=  -   \frac{\tell_t^2(I_t) p(I_t) }{(w^\intercal p(I_t))^2} \\ 
\nabla^2 \tilde{f}_t(w) &=  2  \frac{\tell_t^2(I_t) p(I_t) p(I_t)^\intercal }{(w^\intercal p(I_t))^3}.
\end{align} 
\begin{proofarg}{of Theorem \ref{thm:main-partial-info}}
\textbf{\\Pseudo-regret.} Under the partial information setting, the exp-concavity looks as follows:
\begin{equation*}
\nabla \tilde{f}_t(w) \nabla \tilde{f}_t(w)^\intercal = \frac{\tell_t^4(I_t)p(I_t)p(I_t)^\intercal}{(w^\intercal p(I_t))^4}  \preceq   \frac{n^2L}{2\gamma^2} \cdot \frac{2\tell_t^2(I_t)p(I_t)p(I_t)^\intercal}{(w^\intercal p(I_t))^3}.
\end{equation*}
where the last inequality we used that $w^\intercal p(I_t) \geq \gamma / n$ in $\Delta_k'$ and also that $\tell_t^2(I_t) = \ell^2_t(I_t) / (w_t^\intercal p(I_t) ) \leq nL /\gamma$. Note that the last term in the equation last is the Hessian, so $ \tilde{f}_t$ is $ \frac{2\gamma^2}{n^2L}$-exp-concave.  As for the gradient norm bound, we have,
\begin{equation*}
\norm{\nabla \tilde{f}_t(w)} = \norm { \frac{\tell_t^2(I_t) p(I_t) }{(w^\intercal p(I_T))^2} } \leq \frac{L n^3}{\gamma^3}\norm { p(I_t)} \leq  \frac{L n^2 c \sqrt{k}}{\gamma^3},
\end{equation*}
where the last inequality uses Assumption 1.
These results combined with the ONS regret bound in Equation \ref{eq:ons-bound} provide the following result on the regret in the restricted simplex,
\begin{equation} \label{eq:partial-info-aux-1}
 \frac{1}{n^2} \expval{ \sumtt \tilde{f}_t(w_t) -   \min_{w \in \Delta'_k}\sumtt  \tilde{f}_t(w) } \leq \frac{10 L k^{3/2} c \log T}{\gamma^3}.
\end{equation}
As for the cost of playing in the restricted simplex, it is easy to see that, analogously to the proof of Lemma \ref{lemma:b-bound}, 
\begin{equation}  \label{eq:partial-info-aux-2}
 \frac{1}{n^2}   \expval{ \min_{w \in \Delta'_k}\sumtt   \tilde{f}_t(w) - \min_{w \in \Delta_k}\sumtt   \tilde{f}_t(w)}  \leq \frac{\gamma}{n^2} \sumtt \sum_{i \in \mathbb{I}_+ } \frac{\expval{\tilde{\ell}_t^2(I_t)} }{w_*'^\intercal p(i)}  \leq \frac{\gamma L}{n^2} \sumtt \sum_{i \in \mathbb{I}_+ } \frac{1} {w_*'^\intercal p(i)}  \leq \gamma L T.
\end{equation}
Combining Equations \ref{eq:partial-info-aux-1} and \ref{eq:partial-info-aux-2}, we have a bound on the pseudo-regret
$$\frac{1}{n^2} \expval{ \sumtt \tilde{f}_t(w_t) -   \min_{w \in \Delta_k}\sumtt  \tilde{f}_t(w) } \leq \frac{10 L k^{3/2} c \log T}{\gamma^3} +  \gamma L T. $$

\textbf{\\Expected regret.}
Now we provide guarantees on the expected regret if the adversary is \emph{non-oblivious}, i.e., he can adapt the losses based on past choices of the player. 
We can decompose the regret as follows,
\begin{align} \label{eq:MasterNonOblivious} 
\expval{ \regret_T } 
&=  \ \frac{1}{n^2}\expval{   \sumtt \tf_t(w_t)    - \min_{w \in \Delta_k }  \sumtt f_t(w)}  \nonumber \\
& =\underset{\textup{Pseudo-regret}}{ \underbrace{ \frac{1}{n^2}\expval{   \sumtt \tf_t(w_t)  - \min_{w \in \Delta_k }  \sumtt \tf_t(w)}}}
+\frac{1}{n^2}\expval{ \underset{*}{ \underbrace{\min_{w \in \Delta_k }  \sumtt \tf_t(w)  - \min_{w \in \Delta_k }  \sumtt f_t(w)}}} ,
\end{align}
where the first term is the pseudo-regret we analyzed previously. For bounding $(*)$, we have:
\begin{align}
(*):& = \min_{w \in \Delta_k }  \sumtt \tf_t(w)  - \min_{w \in \Delta_k }  \sumtt f_t(w) \nonumber \\
& \leq  \min_{w \in \Delta_k' }  \sumtt \tf_t(w)  - \min_{w \in \Delta_k }  \sumtt f_t(w)  \nonumber \\
&= \min_{w \in \Delta_k' }  \sumtt \tf_t(w)  - \min_{w \in \Delta'_k }  \sumtt f_t(w) +    \underset{(**)}{ \underbrace{\min_{w \in \Delta'_k }  \sumtt f_t(w) - \min_{w \in \Delta_k }  \sumtt f_t(w)}}   \nonumber \\
\label{eq:star-bound}  & \leq \min_{w \in \Delta_k' }  \sumtt \tf_t(w)  - \min_{w \in \Delta'_k }  \sumtt f_t(w) +    \gamma n^2 LT ,
\end{align}
where in the second line we relied on the definition of the restricted simplex and in bounding $(**)$ we relied on Lemma \ref{lemma:b-bound}.
Denoting   $\tilde{w}_* = \arg \min_{w \in \Delta_k' }  \sumtt \tf_t(w)$ and  $w_* = \arg \min_{w \in \Delta_k' }  \sumtt f_t(w)$, we get a trivial bound on $(*)$ by 
\begin{equation} \label{eq:trivial-exp-regret-bound}
(*) \leq \sumtt \tf_t(\tilde{w}_* ) -  \sumtt f_t(w_*) +  \gamma n^2 LT \leq  \sumtt \tf_t(\tilde{w}_* )  +  \gamma n^2 LT \leq \frac{n^2 L T}{\gamma^2}+  \gamma n^2 LT .
\end{equation}

However, we can achieve tighter bound w.h.p., if we further bound the difference. Denote $\mathbb{I}_+= \{ i | i \in [n], \tell_{1:T}^2(i) -   \ell_{1:T}^2(i) \geq 0  \}$,  $\mathbb{I}_- = \{1,...,n\} \setminus \mathbb{I}_+$ and $\ell_{1:t}^2(i) =  \sum_{\tau=1}^{t} \ell_\tau^2(i)$. We now have:
\begin{align}
\min_{w \in \Delta_k' }  \sumtt \tf_t(w)  - \min_{w \in \Delta_k' }  \sumtt f_t(w) & =  \sumtt \tf_t(\tilde{w}_* ) -  \sumtt f_t(w_*)  \nonumber \\
& \leq  \sumtt \tf_t(w_* ) -  \sumtt f_t(w_*) \nonumber \\
& =  \sum_{i \in \mathbb{I}_+} \frac{( \tell_{1:T}^2(i) -  \ell_{1:T}^2(i)  )}{w_*^\intercal p(i)} +  \sum_{i \in \mathbb{I}_-} \frac{( \tell_{1:T}^2(i) -  \ell_{1:T}^2(i)  )}{w_*^\intercal p(i)}\nonumber \\
& \leq  \sum_{i \in \mathbb{I}_+} \frac{( \tell_{1:T}^2(i) -  \ell_{1:T}^2(i)  )}{w_*^\intercal p(i)} \nonumber \\
& \leq \frac{n}{\gamma} \left(\sum_{i \in \mathbb{I}_+}   \tell_{1:T}^2(i) -  \sum_{i \in \mathbb{I}_+}  \ell_{1:T}^2(i)  \right), \label{eq:l-bound}
\end{align}
where the second line  uses the definition of $\tilde{w}_*$,  the fourth line we discards the negative terms of the summation over $\mathbb{I}_-$    and the last inequality relies on the fact that $w_*^\intercal p(i) \geq \gamma / n$ for all $i \in [n]$ in the restricted simplex.  For brevity and w.l.o.g. assume that $ \mathbb{I}_+ = \{1,...,n\}$.
Define the following sequence $\{Z_{t}: =\sumin  \tell_t^2(i) - \sumin \ell_t^2(i) \}_{t\in[T]}$.   $\{ Z_{t}\}_{t\in[T]}$ is a martingale difference sequence with respect to the filtration $\{\F_{t}\}_{t\in[T]}$ associated with the history of the strategy, since  $\mathbb{E}[\sumin\tell_t^2(i) \vert w_t,\ell_t] = \sumin \ell_t^2(i)$. Due to the restricted simplex, we have $$|Z_{t}| \leq |\sumin  \tell_t^2(i) | +| \sumin  \ell_t^2(i) | = \left| \frac{\ell_t^2(I_t)}{w_t ^\intercal p(I_t)} \right| +| \sumin  \ell_t^2(i) |  \leq   \left| \frac{n \ell_t^2(I_t)}{\gamma} \right| +   nL  \leq \frac{2nL}{\gamma}.$$ The conditional variance of the  $Z_{t}$ can be bounded as follows,
 \begin{align}
 \Var(Z_{t}\vert \F_{t-1})
 & = 
\expval{ \left( \sumin \frac{\ell_t^2(i)}{w_t^\intercal p(i)}\mathbbm{1} _{I_t=i} - \sumin \ell_t^2(i)   \right)^2\vert \F_{t-1}}  \nonumber\\
 & = 
\expval{ \sumin \frac{\ell_t^4(i)}{(w_t^\intercal p(i))^2}\mathbbm{1} _{I_t=i} -2\left( \sumin \frac{\ell_t^2(i)}{w_t^\intercal p(i)}\mathbbm{1} _{I_t=i} \right)\cdot \left(\sumin \ell_t^2(i) \right) + \left( \sumin \ell_t^2(i)  \right)^2  \vert \F_{t-1}}  \nonumber\\
 &=
  \sumin \frac{\ell_t^4(i)}{w_t^\intercal p(i)} -  \left( \sumin \ell_t^2(i)  \right)^2 \nonumber\\
 &\leq
 L \sumin \frac{\ell_t^2(i)}{w_t^\intercal p(i)}  \label{eq:var-bound}.
 \end{align}
Having a bounded martingale difference sequence at hand, we can use Freedman's lemma.
 
 \begin{lemma}[Freedman's Inequality \citep{freedman1975tail, kakade2009generalization}]\label{lemma:freedman}
Suppose  $\{Z_t\}_{t\in[T]}$ is a martingale difference sequence with respect to a filtration $\{\F_t\}_{t\in[T]}$, such that $|Z_t|\leq b$. Define
$
\Var_t Z_t = \Var\left(Z_t \vert \F_{t-1}\right)
$
and  let $\sigma= \sqrt{\sum_{t=1}^T\Var_t Z_t} $ be the sum of conditional variances of $Z_t$'s. Then for any $\delta\leq 1/e$ and 
$T\geq 3$ we have,
$$
P\left( \sum_{t=1}^TZ_t \geq  \max\left\{2\sigma, 3b\sqrt{\log(1/\delta)} \right\}\sqrt{\log(1/\delta)} 
 \right) \leq  4\delta\log(T).
$$
\end{lemma}
 An immediate corollary of Freedman's lemma applied to our setting is that for all  $t \in [T]$ with probability $\geq 1 - 4T\delta \log(T)$ we have
 \begin{equation*}
\sumtt Z_{t}  \leq (2\sigma + 3b) \log (1 / \delta),
 \end{equation*}
 which is a result of the union bound. For simplicity, we ignore the $ \log (1 / \delta)$ factor as we choose $\delta = 1/\textup{poly}(T)$ and thus $\log (1 / \delta)$  has a logarithmic contribution to the regret. Using the definition of $\sigma$ and $b$ we have w.h.p.
 \begin{equation}
  \sumin \tell_{1:T}^2(i) - \sumin  \ell_{1:T}^2(i)   \leq 2 \sqrt{\sumtt  \sumin  L\frac{\ell_t^2(i)}{w_t^\intercal p(i)}} + \frac{6nL}{\gamma} \leq 2 \sqrt{\frac{n^2L^2T}{\gamma}} + \frac{6nL}{\gamma}.
 \end{equation}
Plugging this result into Equation \ref{eq:l-bound}, we get w.h.p.
\begin{equation} \label{eq:whp-bound}
\min_{w \in \Delta_k' }  \sumtt \tf_t(w)  - \min_{w \in \Delta_k' }  \sumtt f_t(w) \leq  \frac{n^2L}{\gamma} \left(\frac{2\sqrt{T}}{\sqrt{\gamma}} + \frac{6}{\gamma} \right).
\end{equation}

Since Equation \ref{eq:trivial-exp-regret-bound} provides a trivial bound  and Equation \ref{eq:whp-bound} gives a h.p. bound, we can choose $\delta = 1/\textup{poly}(T)$ small enough such that, combined with Equations \ref{eq:MasterNonOblivious}, \ref{eq:star-bound} and  \ref{eq:partial-info-aux-1}, we have almost surely:
\begin{equation}
\expval{\regret_T} = L \cdot  \tilde{\mathcal{O}} \left( \underset{\textup{ONS restricted simplex}}{\underbrace{\frac{k^{3/2}c }{\gamma^3}}} + \underset{\textup{mixing}}{\underbrace{\gamma T }}  +  \underset{\textup{non-oblivious}}{\underbrace{\frac{T^{1/2}}{\gamma^{3/2}} + \frac{1}{\gamma^2}}} \right)
\end{equation} 
We set $\gamma = k^{3/8} c^{1/5} T^{-1/5}$, use that $c \leq T$ and get  
\begin{equation}
\expval{\regret_T} = \tilde{\mathcal{O}} \left(  k^{3/8} c^{1/5}  L T^{4/5} \right).
\end{equation} 
\end{proofarg}
\end{section}

\begin{section}{Dataset Details}
\begin{itemize}
\setlength\itemsep{0.2em}
  \item \texttt{CSN} \citep{faulkner2011next} --- $n=80\,000$, $d = 17$; cellphone accelerometer data
 \item \texttt{KDD} \citep{kddcup2004} --- $n=145\,751$, $d = 74$; Protein Homology Prediction KDD competition  dataset
    \item \texttt{MNIST} \citep{lecun1998gradient} --- $n=70\,000$, $d = 10$; the original low resolution images of handwritten
    characters are transformed using PCA with whitening and 10 principal components are retained
\end{itemize}
\end{section}

\end{document}